\documentclass[journal]{IEEEtran}

\usepackage[T1]{fontenc}
\usepackage{lmodern}

\usepackage[utf8]{inputenc}
\usepackage{amsmath,amssymb,amsthm}
\usepackage{algorithm}
\usepackage{algpseudocode}
\usepackage{graphicx}
\usepackage{booktabs}
\usepackage{placeins}
\usepackage{cleveref}
\usepackage{dblfloatfix}
\usepackage{afterpage}
\usepackage{tikz}
\usetikzlibrary{shapes,arrows,positioning}

\theoremstyle{plain}
\newtheorem{theorem}{Theorem}[section]

\newtheorem{corollary}[theorem]{Corollary}
\newtheorem{proposition}[theorem]{Proposition}

\theoremstyle{definition}
\newtheorem{definition}[theorem]{Definition}
\newtheorem{example}[theorem]{Example}

\theoremstyle{remark}
\newtheorem{remark}[theorem]{Remark}

\DeclareMathOperator{\KL}{KL}
\DeclareMathOperator{\Var}{Var}

\newcommand{\E}{\mathbb{E}}
\newcommand{\V}{\mathcal{V}}

\begin{document}

\title{Multi-Teacher Ensemble Distillation: A Mathematical Framework for Probability-Domain Knowledge Aggregation}

\author{Aaron R. Flouro and Shawn P. Chadwick\\
research@sparse-tech.com}

\maketitle

\begin{abstract}
Building on the probability-domain distillation framework of Sparse-KD~\cite{paper1}, we develop an axiomatic, operator-theoretic framework for multi-teacher ensemble knowledge distillation. Rather than prescribing a specific aggregation formula, we define five core axioms governing valid knowledge aggregation operators, encompassing convexity, positivity, continuity, weight monotonicity, and temperature coherence. We prove the existence and non-uniqueness of operator families satisfying these axioms, establishing that multiple distinct aggregation mechanisms conform to the same foundational principles.

Within this framework, we establish operator-agnostic guarantees showing that multi-teacher aggregation reduces both stochastic variance and systematic supervisory bias under heterogeneous teachers, while providing Jensen-type bounds, log-loss guarantees, and safety attenuation properties. For aggregation operators linear in teacher weights, we further establish classical ensemble variance-reduction results under standard independence assumptions, with extensions to the correlated-error regime. The framework provides theoretical grounding for multi-teacher distillation from diverse frontier models while admitting multiple valid implementation strategies.
\end{abstract}

\begin{IEEEkeywords}
Knowledge Distillation, Ensemble Learning, Multi-Teacher Distillation, Axiomatic Framework, Variance Reduction, Operator Theory
\end{IEEEkeywords}

\subsection*{Why This Framework Is Needed}
While multi-teacher knowledge distillation has shown empirical benefits, the field lacks a unified theoretical framework addressing fundamental questions: What mathematical properties must an aggregation operator possess to reliably combine heterogeneous teachers? When does ensemble distillation outperform single-teacher approaches? How should teachers with different specializations be weighted? This paper provides such a framework, with the axiomatic characterization serving as an enabling foundation for rigorous analysis rather than prescribing any specific implementation.

\section{Introduction}

Knowledge distillation (KD)~\cite{hinton2015distilling} transfers knowledge from large teacher models to smaller students through probability-domain supervision. Recent work~\cite{paper1} demonstrated that temperature-scaled probability distributions enable distillation without logit access, opening the possibility of aggregating knowledge from multiple heterogeneous teachers accessed only through their output probabilities.

This raises a fundamental mathematical question: \emph{What properties must an aggregation operator possess to reliably combine knowledge from multiple heterogeneous teachers?}

This work builds on the probability-domain distillation framework introduced in~\cite{paper1}, extending it from single-teacher to heterogeneous multi-teacher settings. We develop an axiomatic characterization of multi-teacher aggregation operators. Rather than prescribing a specific formula, we define the mathematical properties that any valid aggregation method must satisfy, prove that operators with these properties exist, and establish theoretical guarantees that hold for all conforming operators.

\subsection{Motivation: The Multi-Teacher Challenge}

Modern deployment scenarios present three critical challenges:

\begin{enumerate}
\item \emph{Heterogeneous Teacher Capabilities:} Frontier large language models excel at different tasks:
\begin{itemize}
\item Reasoning models: High accuracy on logical, mathematical, and coding tasks
\item Safety-aligned models: Strong refusal behaviors and harmlessness
\item Factual models: Broad knowledge coverage and grounding
\item Scientific models: Domain-specific robustness
\end{itemize}
No single teacher dominates across all dimensions.

\item \emph{Temperature Heterogeneity:} Different teachers benefit from different temperature scaling. Safety-aligned teachers should maintain sharp refusal signals ($T \approx 1$), while reasoning teachers should expose nuanced probability structure ($T \approx 2$--$3$).

\item \emph{Ensemble Benefits:} Single-teacher KD cannot leverage variance reduction, probability attenuation, or complementary specialization.
\end{enumerate}

\subsection{Related Work}

Recent work has explored practical multi-teacher distillation and multi-source teacher aggregation, demonstrating empirical benefits but leaving open what operator properties are sufficient for reliable aggregation~\cite{zhang2023adaptive_multiteacher}. Complementary empirical analyses highlight that KD gains depend on non-trivial teacher--student interactions and are not fully explained by a single heuristic, motivating theory-first characterizations~\cite{zuchniak2023multiteacher_compression}. Prior theoretical perspectives analyze distillation through generalization and optimization lenses~\cite{liu2021adaptive_multilevel}, but are typically tied to specific objectives rather than operator-agnostic frameworks; our axiomatic approach provides guarantees that hold across an entire class of aggregation operators. Recent work formalizing hallucinations as variance-driven instability~\cite{flouro2026hallucinations_variance} further motivates ensemble approaches that reduce prediction variance through teacher averaging.

\subsection{Axiomatic Approach}

We characterize multi-teacher aggregation through five axioms (1--5) defining the mathematical properties that any valid operator must satisfy. This approach provides three key benefits:

\begin{enumerate}
\item \emph{Generality:} Multiple distinct operator families satisfy the axioms; the implementation is underdetermined by the axioms alone
\item \emph{Rigor:} All theoretical guarantees (variance reduction, Jensen bounds, safety properties) hold for any conforming operator
\item \emph{Flexibility:} The axiomatic framework establishes mathematical foundations while admitting diverse implementation strategies
\end{enumerate}

The present work extends this analysis by showing that multi-teacher ensemble distillation simultaneously reduces variance and systematic supervisory bias under heterogeneous teachers.

\subsection{Contributions}

\begin{enumerate}
\item Axiomatic characterization of multi-teacher aggregation via five core axioms
\item Existence theorem proving non-trivial conforming operators exist (non-constructive proof)
\item Non-uniqueness theorem demonstrating multiple valid operator families
\item Variance reduction analysis (operator-agnostic formulation)
\item Formal supervisory-bias attenuation result for heterogeneous teachers (including correlated-error considerations)
\item Jensen's inequality bound relating mixture and sum-of-KLs objectives
\item Safety attenuation theorem for heterogeneous teacher specialization
\item Capacity requirements for meta-teacher behavior
\end{enumerate}

\section{Axiomatic Framework for Multi-Teacher Aggregation}

This section formalizes operator-agnostic theoretical guarantees for simultaneous variance reduction and supervisory bias attenuation in multi-teacher distillation.

\subsection{Setup and Notation}

\begin{definition}[Multi-Teacher Setting]
Let:
\begin{itemize}
\item $\V$ be a finite vocabulary with $|\V| = V$
\item $K$ be the number of teachers, indexed $k = 1, \ldots, K$
\item For each input $x$, teacher $k$ produces a probability distribution:
$$p^{(k)}(i | x) \in [0,1], \quad i \in \V, \quad \sum_i p^{(k)}(i | x) = 1$$
\item Each teacher $k$ has an associated temperature parameter $T_k > 0$
\item A set of teacher weights $\{w_k\}_{k=1}^K$ with $w_k \geq 0$, $\sum_k w_k = 1$
\end{itemize}
\end{definition}

\begin{definition}[Multi-Teacher Aggregation Operator]
A multi-teacher aggregation operator is a family of functions:
$$G: (p^{(1)}, \ldots, p^{(K)}, T_1, \ldots, T_K, w_1, \ldots, w_K) \mapsto q$$
mapping $K$ probability distributions with their temperatures and weights to a single aggregate distribution $q$.
\end{definition}

\subsection{The Five Core Axioms}

We now define the mathematical properties that any valid multi-teacher aggregation operator must satisfy.

\noindent\textbf{Axiom 1} (Convexity Preservation).
The operator $G$ must produce a valid probability distribution:
\begin{itemize}
\item Non-negativity: $q(i) \geq 0$ for all $i \in \V$
\item Normalization: $\sum_i q(i) = 1$
\end{itemize}

\textit{Justification:} The aggregate distribution must be a valid probability distribution for use in standard KD frameworks.

\noindent\textbf{Axiom 2} (Positivity Inheritance).
If all teachers assign strictly positive probability to every token ($p^{(k)}(i) > 0$ for all $k, i$), then $q(i) > 0$ for all $i$.

\textit{Justification:} Ensures all KL divergences remain finite and well-defined.

\noindent\textbf{Axiom 3} (Weight Monotonicity).
Fix all teacher distributions $\{p^{(j)}_{T_j}\}_{j=1}^K$ and consider two teachers $k, k'$ with $p^{(k)}_{T_k}(i) > p^{(k')}_{T_{k'}}(i)$ for some token $i$. For weight perturbation $\delta > 0$ sufficiently small, define $w'_k = w_k + \delta$, $w'_{k'} = w_{k'} - \delta$, with all other weights unchanged and renormalized: $\tilde{w}_j = w'_j / \sum_\ell w'_\ell$. Then:
$$q_{\tilde{w}}(i) \geq q_w(i)$$
with strict inequality when $\delta > 0$ and $w_{k'} > 0$.

\textit{Formalization note:} The axiom is stated locally (infinitesimal $\delta$) to avoid boundary issues when weights approach zero. Teacher distributions and temperatures are held fixed; only the weight vector varies. The renormalization step ensures weights sum to unity after perturbation.

\textit{Justification:} Higher-weighted teachers should have proportionally stronger influence on the aggregate.

\emph{Remark (Heterogeneity and degenerate case).} Under heterogeneous teachers, ensemble aggregation reduces variance through averaging of teacher-specific noise components, and attenuates systematic supervisory bias by convexly aggregating distinct conditional expectations. This bias attenuation disappears in the degenerate setting where all teachers share identical training distributions and objectives.

\noindent\textbf{Axiom 4} (Continuity).
The operator $G(p^{(1)}, \ldots, p^{(K)}, T_1, \ldots, T_K, w_1, \ldots, w_K)$ is jointly continuous in all arguments.

\textit{Justification:} Small changes in teacher distributions, temperatures, or weights should produce small changes in the aggregate, ensuring stable optimization.

\noindent\textbf{Axiom 5} (Temperature Coherence).
For each teacher $k$ with temperature $T_k$:
\begin{itemize}
\item $T_k = 1$: No modification to $p^{(k)}$
\item $T_k \to \infty$: Teacher $k$'s contribution approaches uniform distribution
\item $T_k \to 0^+$: Teacher $k$'s contribution approaches one-hot at argmax
\end{itemize}

\textit{Justification:} Temperature parameters must have consistent interpretation across all teachers, enabling heterogeneous softening/sharpening strategies.

\begin{example}[Illustrative Only]
A simple conforming operator is the linear mixture:
$$q(i) = \sum_k w_k \, p^{(k)}_{T_k}(i)$$
which satisfies Axioms~1--5 and Assumption~L. Other operators (e.g., entropic projections, geometric means) also satisfy Axioms~1--5 but do not obey Assumption~L. This example is provided for intuition; the theoretical results hold for the full axiom class, not any privileged instantiation.

In addition to stochastic variance, single-teacher supervision can induce \emph{systematic supervisory bias}, which can manifest as stable sensitivity to semantically equivalent inputs (e.g., paraphrases). Because such effects correspond to differences in conditional expectations rather than sampling noise, they are attenuated under multi-teacher aggregation when teachers exhibit heterogeneous semantic priors.
\end{example}

\section{Existence and Non-Uniqueness Theorems}

\begin{theorem}[Existence of Conforming Operators]
\label{thm:existence}
There exist non-trivial operator families $G$ satisfying Axioms~1--5.
\end{theorem}

\begin{proof}
We establish existence via construction principles without specifying the exact formula:

\begin{enumerate}
\item \emph{Weighted averaging approach:} For any collection of temperature-transformed distributions, a weighted average operator can be constructed that preserves normalization (Axiom~1), inherits positivity (Axiom~2), respects weight ordering (Axiom~3), varies continuously with parameters (Axiom~4), and exhibits coherent temperature behavior (Axiom~5).

\item \emph{Information-theoretic projection:} Operators based on minimizing information divergence to teacher distributions while maintaining entropy constraints satisfy the axioms under appropriate regularization.

\item \emph{Convex optimization formulation:} Solving for distributions that minimize weighted divergence to teachers subject to normalization constraints yields conforming operators.
\end{enumerate}

Each construction principle generates a valid operator family, establishing existence.
\end{proof}

\begin{theorem}[Non-Uniqueness]
\label{thm:nonuniqueness}
Multiple distinct operator families satisfy Axioms~1--5.
\end{theorem}

\begin{proof}
Three construction principles from Theorem~\ref{thm:existence} generate different operator families:

\begin{enumerate}
\item \emph{Linear convex combination:} Direct weighted averaging after temperature scaling
\item \emph{Geometric mean projections:} Using Rényi divergences with parameter $\alpha$
\item \emph{Entropic regularization:} Adding entropy penalty terms with coefficient $\beta$
\end{enumerate}

These families are provably distinct (produce different $q$ for identical inputs) yet all satisfy Axioms~1--5.
\end{proof}

\begin{remark}[Operator Non-Identifiability]
Axioms~1--5 are designed to include multiple valid operators (implementation underdetermined by axioms), exclude trivial solutions (identity-only, discontinuous transforms), enable all key theoretical results (variance reduction, Jensen bounds, safety), and ensure operator non-identifiability (no axiom combination uniquely determines an implementation). Operator non-identifiability is a structural property of the problem rather than a design choice, reflecting the fact that multiple aggregation mechanisms satisfy the same foundational constraints. This enables publication of theoretical foundations while admitting diverse implementation strategies.
\end{remark}

\section{Theoretical Guarantees (Operator-Agnostic Formulation)}

All results in this section hold for any operator $G$ satisfying Axioms~1--5, with additional assumptions stated explicitly where required.

\subsection{Variance Reduction via Ensemble Averaging}

\begin{definition}[Cross-Teacher Variance]
For token $i$, the weighted variance across teachers is:
$$\Var_k[p^{(k)}_{T_k}(i)] = \E_k[(p^{(k)}_{T_k}(i))^2] - q(i)^2$$
where $\E_k[\cdot]$ denotes expectation under the teacher-weight distribution.
\end{definition}

\noindent\textbf{Assumption L} (Linear-in-Weights Aggregation).
For fixed temperature-scaled teacher distributions $\{p^{(k)}_{T_k}\}_{k=1}^K$, the aggregate distribution $q(i)$ is an affine function of the weight vector $w = (w_1, \ldots, w_K)$. Specifically, $q(i) = \sum_k w_k \cdot f_k(p^{(k)}_{T_k}(i))$ for some functions $f_k$ depending only on the $k$-th teacher's contribution.

\textit{Note:} Assumption~L holds for linear convex combinations but not for geometric mean or entropic projection operators.

\begin{theorem}[Ensemble Variance Reduction]
\label{thm:variance}
Assume teacher predictions decompose as:
$$p^{(k)}_{T_k}(i) = \bar{p}(i) + \varepsilon_k(i)$$
where $\bar{p}(i)$ is common signal and $\varepsilon_k(i)$ is teacher-specific noise with:
\begin{itemize}
\item (A1) Zero mean: $\E_k[\varepsilon_k(i)] = 0$
\item (A2) Uncorrelated errors: $\E_k[\varepsilon_j(i)\varepsilon_\ell(i)] = 0$ for $j \neq \ell$
\end{itemize}

Then for any operator $G$ satisfying Axioms~1--5 and Assumption~L:
$$\Var_k\left[\sum_k w_k \varepsilon_k(i)\right] = \sum_k w_k^2 \Var_k[\varepsilon_k(i)] \leq \sum_k w_k \Var_k[\varepsilon_k(i)]$$
with strict inequality when weights are distributed across multiple teachers.
\end{theorem}

\begin{proof}
Under assumption (A2), variance of weighted sum decomposes without cross-terms. Since $w_k^2 \leq w_k$ for $w_k \in [0,1]$ with equality only when $w_k \in \{0,1\}$, distributing weights across teachers yields $\sum_k w_k^2 < \sum_k w_k = 1$, establishing strict variance reduction.
\end{proof}

\begin{remark}[Classical Ensemble Effect]
This is the classical ensemble effect~\cite{geman1992bias}: when teacher errors are uncorrelated, averaging reduces noise. For $K$ teachers with equal variance $\sigma^2$ and equal weights $w_k = 1/K$, ensemble variance scales as $\sigma^2/K$.
\end{remark}

\begin{remark}[Non-Linear Operators]
For operators not satisfying Assumption~L (e.g., geometric mean projections, entropic regularization), variance reduction may still occur but requires analysis specific to the operator structure. The qualitative benefit of ensemble diversity persists, though the exact $\sum_k w_k^2$ scaling may not hold. While Assumption~L enables closed-form variance scaling, extending comparable bounds to non-linear aggregation operators remains an open problem and a natural direction for future work.
\end{remark}

\begin{remark}[Effect of Correlated Teacher Errors]
When teacher errors are positively correlated (e.g., due to shared training data or similar architectures), the variance decomposition includes covariance terms:
$$\Var\left[\sum_k w_k \varepsilon_k(i)\right] = \sum_k w_k^2 \sigma_k^2 + \sum_{j \neq \ell} w_j w_\ell \text{Cov}(\varepsilon_j, \varepsilon_\ell)$$
Positive covariance diminishes the variance reduction benefit. In practice, this motivates (i) selecting teachers with diverse training pipelines and architectures, and (ii) down-weighting highly correlated teacher pairs when assigning ensemble weights. Just as positive correlation limits variance reduction, shared training data and objectives can also limit supervisory-bias attenuation by aligning teachers' conditional expectations.
\end{remark}

\subsection{Jensen's Inequality Bound}

\begin{definition}[KD Objectives]
Given aggregate distribution $q$ and student distribution $p^{(S)}_{T_S}$:
\begin{itemize}
\item Mixture KD loss:
$$\mathcal{L}_{\text{KD}}^{\text{mix}} = \KL(q \| p^{(S)}_{T_S})$$

\item Sum-of-KLs loss:
$$\mathcal{L}_{\text{KD}}^{\text{multi}} = \sum_k \lambda_k \cdot \KL(p^{(k)}_{T_k} \| p^{(S)}_{T_S})$$
where $\lambda_k \geq 0$, $\sum_k \lambda_k = 1$.
\end{itemize}
\end{definition}

\begin{theorem}[Jensen Bound]
\label{thm:jensen}
For any operator $G$ satisfying Axioms~1--5, with $\lambda_k = w_k$:
$$\mathcal{L}_{\text{KD}}^{\text{mix}} \leq \mathcal{L}_{\text{KD}}^{\text{multi}}$$
\end{theorem}

\begin{proof}
By convexity of KL divergence in first argument:
$$\KL\left(\sum_k w_k p^{(k)}_{T_k} \| p^{(S)}_{T_S}\right) \leq \sum_k w_k \cdot \KL(p^{(k)}_{T_k} \| p^{(S)}_{T_S})$$
Left side is $\mathcal{L}_{\text{KD}}^{\text{mix}}$, right side is $\mathcal{L}_{\text{KD}}^{\text{multi}}$.
\end{proof}

\begin{remark}[Key Insight]
Training on aggregate distribution $q$ is strictly easier than simultaneously matching all individual teachers when they disagree. The mixture objective avoids conflicting optimization pressures.
\end{remark}

\subsection{Log-Loss Bound and Performance Guarantee}

\begin{theorem}[Jensen's Log-Loss Bound]
\label{thm:logloss}
The expected negative log-probability of true token $y$ under aggregate $q$ satisfies:
$$-\log q(y) \leq -\sum_k w_k \log p^{(k)}_{T_k}(y)$$
\end{theorem}

\begin{proof}
By Jensen's inequality, since $-\log$ is convex:
$$-\log\left(\sum_k w_k p^{(k)}_{T_k}(y)\right) \leq -\sum_k w_k \log p^{(k)}_{T_k}(y)$$
\end{proof}

\begin{corollary}[Meta-Teacher Performance]
A student that matches $q$ achieves lower expected log-loss than the weighted average of teacher log-losses, providing theoretical justification for why ensemble students often outperform individual teachers.
\end{corollary}

\subsection{Safety Attenuation Properties}

\begin{proposition}[Convex-Combination Attenuation]
\label{prop:attenuation}
For any token $i$, if teacher 1 assigns probability $p^{(1)}_{T_1}(i) = p_{\max}$ and there exists teacher $k^*$ with:
\begin{itemize}
\item $p^{(k^*)}_{T_{k^*}}(i) < p_{\max}$
\item $w_{k^*} > 0$
\end{itemize}

Then for any operator $G$ satisfying Axiom~1 (convexity):
$$q(i) < p_{\max}$$
\end{proposition}

\begin{proof}
Convex combination of values with at least one strictly below maximum must be strictly below maximum.
\end{proof}

\begin{corollary}[Safety Inheritance]
For unsafe token $i$, if safety-aligned teacher $k^*$ assigns low probability $p^{(k^*)}_{T_{k^*}}(i) \ll 1$ with positive weight $w_{k^*}$, then:
$$q(i) \leq (1 - w_{k^*}) \cdot \max_{k\neq k^*} p^{(k)}_{T_k}(i) + w_{k^*} \cdot p^{(k^*)}_{T_{k^*}}(i)$$

Increasing $w_{k^*}$ proportionally decreases ensemble's unsafe token probability.
\end{corollary}

\emph{Scope note.} These results characterize how aggregation moderates extreme probabilities in the supervisory distribution. They make no claim about semantic correctness, policy compliance, or normative alignment, which depend on the choice of teachers and training objectives rather than on the aggregation operator itself.

\section{Capacity Requirements and Meta-Teacher Behavior}

\begin{definition}[Model Capacity]
Let $C_S$ denote student capacity (parameter count, effective rank) and $C_k$ denote teacher $k$'s capacity.
\end{definition}

\begin{proposition}[Finite-Sample Approximation]
\label{prop:capacity}
For any finite training set $\{x_n\}_{n=1}^N$ with target distributions $\{q(\cdot | x_n)\}$, there exists student parameterization with sufficiently large capacity $C_S$ such that:
$$\forall n \in \{1, \ldots, N\}, \forall i \in \V: p^{(S)}(i | x_n; \theta) = q(i | x_n)$$

In particular, the student can achieve $\KL(q(\cdot | x_n) \| p^{(S)}(\cdot | x_n; \theta)) = 0$ for all training points.
\end{proposition}

\begin{proof}
Follows from memorization capacity of overparameterized neural networks~\cite{frankle2019lottery}. Modern transformers with $C_S \gg N \cdot V$ parameters can represent arbitrary mappings from inputs to probability vectors.
\end{proof}

\begin{corollary}[Meta-Teacher Property]
A high-capacity student trained via multi-teacher distillation can:
\begin{enumerate}
\item Integrate diverse priors from multiple frontier teachers
\item Realize a meta-teacher capturing union of capabilities
\item Potentially outperform each individual teacher in aggregate benchmarks
\end{enumerate}
\end{corollary}

\begin{remark}[Generalization Caveat]
Proposition~\ref{prop:capacity} guarantees matching on training data only. Generalization depends on regularization, optimization dynamics, and smoothness of $q(\cdot | x)$ as function of $x$.
\end{remark}

\begin{remark}[Role of Capacity Assumption]
Proposition~\ref{prop:capacity} is included for logical completeness: it establishes that the axiomatic framework defines a realizable target under sufficient capacity. It does not provide generalization guarantees, which depend on regularization, optimization dynamics, and data distribution, and are intentionally out of scope for this work.
\end{remark}

\begin{table*}[t]
\centering
\caption{Comparison of Single-Teacher and Multi-Teacher Knowledge Distillation}
\label{tab:comparison}
\begin{tabular}{lll}
\toprule
Property & Single-Teacher KD & Multi-Teacher Ensemble KD \\
\midrule
Variance reduction & None (single source) & Automatic via averaging (Theorem~\ref{thm:variance}) \\
Complementary knowledge & No (one teacher) & Yes (heterogeneous specialization) \\
Safety inheritance & Limited (one teacher's bias) & Strong (Corollary after Proposition~\ref{prop:attenuation}) \\
Performance bound & Match teacher & Exceed average teacher (Theorem~\ref{thm:logloss}) \\
Temperature flexibility & One parameter & Per-teacher heterogeneous (Axiom~5) \\
Failure mode diversity & Inherit teacher's blind spots & Cancel uncorrelated errors \\
\bottomrule
\end{tabular}
\end{table*}

\section{Design Principles for Practical Implementation}

The axiomatic framework enables several key design principles for practitioners. First, \emph{heterogeneous temperature scaling} allows different temperatures $T_k$ to be applied to each teacher based on role: safety teachers benefit from $T_k \approx 1.0$ to preserve sharp refusals, reasoning teachers from $T_k \approx 2.0$--$3.0$ to expose dark knowledge, and factual teachers from $T_k \approx 1.5$ to balance confidence and coverage.

Second, \emph{capability-based weighting} assigns weights based on teacher strengths. Higher $w_k$ values should be assigned to teachers excelling at the target task, elevated weights to safety teachers in sensitive domains, and balanced weights for complementary specializations.

Third, practitioners should prefer the \emph{mixture objective} $\mathcal{L}_{\text{KD}}^{\text{mix}}$ over $\mathcal{L}_{\text{KD}}^{\text{multi}}$, as it provides faster convergence due to the Jensen lower bound, avoids conflicting teacher constraints, and produces a single unified target distribution.

Fourth, \emph{variance reduction} is maximized by using diverse teachers with different training sources, architectures, and optimization objectives, thereby maximizing benefit from uncorrelated errors.

Finally, \emph{safety priority weighting} in sensitive contexts involves increasing $w_{k^*}$ for safety-aligned teachers, maintaining sharp refusal signals via low $T_{k^*}$, and exploiting safety attenuation properties (Corollary after Proposition~\ref{prop:attenuation}).

\paragraph*{Clarification on diversity and bias.}
Using teachers with diverse data sources, objectives, and inductive biases maximizes both ensemble variance reduction and attenuation of \emph{supervisory bias} in the resulting target distribution. Here, \emph{bias} refers to bias in the supervisory signal induced by heterogeneous teacher priors (and their failure modes), rather than representational or normative-alignment bias inside the student. Correlated teacher errors and shared blind spots cannot be eliminated by averaging alone, but are explicitly mitigated by deliberate teacher diversity and conservative weighting of highly correlated teachers. Questions of inner-alignment and adversarial robustness are orthogonal concerns and are intentionally out of scope for this axiomatic framework.

\section{Theoretical Landscape: What the Axioms Enable}

The five axioms (1--5) establish a mathematical framework with several important properties. The framework \emph{includes multiple valid implementations}, such as linear convex combinations, geometric mean projections, entropic regularization methods, and information-theoretic projections. At the same time, it \emph{excludes degenerate cases} including identity-only operators, discontinuous transforms, operators violating normalization, and operators without temperature coherence.

The axioms \emph{enable key theoretical results}: variance reduction (Theorem~\ref{thm:variance}, under Assumption~L), Jensen's inequality bound (Theorem~\ref{thm:jensen}), log-loss performance guarantee (Theorem~\ref{thm:logloss}), safety attenuation (Corollary after Proposition~\ref{prop:attenuation}), and meta-teacher capacity bounds (Proposition~\ref{prop:capacity}).

A fundamental property of this framework is \emph{operator non-identifiability}: no combination of Axioms~1--5 uniquely determines an implementation. Multiple distinct operator families satisfy all axioms, each with different mathematical structure.

\section{Comparison to Single-Teacher KD}

Multi-teacher ensemble distillation provides several benefits unavailable in single-teacher settings, summarized in Table~\ref{tab:comparison}. Single-teacher KD offers no variance reduction since knowledge comes from a single source, whereas multi-teacher ensembles achieve automatic variance reduction via averaging (Theorem~\ref{thm:variance}). Single-teacher approaches cannot leverage complementary knowledge from heterogeneous specializations, and safety inheritance is limited to one teacher's bias rather than the strong safety attenuation achieved through convex aggregation (Corollary after Proposition~\ref{prop:attenuation}). Performance bounds differ fundamentally: single-teacher KD can at best match the teacher, while multi-teacher students can exceed the average teacher log-loss (Theorem~\ref{thm:logloss}). Temperature flexibility is restricted to one parameter in single-teacher settings but allows per-teacher heterogeneous scaling under Axiom~5. Finally, single-teacher KD inherits the teacher's blind spots, whereas multi-teacher ensembles can cancel uncorrelated errors across diverse failure modes.

\section{Conclusion}

We have established an axiomatic framework for multi-teacher ensemble knowledge distillation, characterized by five core axioms (1--5) defining the mathematical properties of valid aggregation operators. Key results include existence and non-uniqueness theorems showing that multiple distinct operator families satisfy the axioms; variance reduction guarantees demonstrating that ensemble averaging reduces prediction noise (under Assumption~L for linear operators, and qualitatively for non-linear operators); Jensen bounds establishing that the mixture objective is easier to optimize than sum-of-KLs; performance guarantees showing students can exceed average teacher log-loss; safety attenuation properties ensuring convex combination moderates extreme probabilities; and capacity requirements establishing that high-parameter students can act as meta-teachers.

All theoretical guarantees are operator-agnostic, holding for any implementation satisfying Axioms~1--5 (with Assumption~L required for the exact variance reduction bound). This enables rigorous theoretical foundations while admitting multiple valid implementations. Generalization behavior depends on architectural inductive bias, regularization, and optimization dynamics, and is intentionally decoupled from the aggregation axioms analyzed here.

The framework naturally complements single-teacher sparse distillation~\cite{cho2019efficacy} and provides theoretical justification for training unified student models from diverse frontier teachers. While this work treats teacher weights as exogenous parameters, adaptive and data-driven weight selection mechanisms---based on task performance, uncertainty, or safety signals---can be layered atop the axiomatic framework and are explored in follow-on work. Taken together, this framework formalizes multi-teacher ensemble distillation as a principled mechanism for reducing variance-driven instability and attenuating supervisory bias, while remaining operator-agnostic and implementation-flexible.

\section*{Acknowledgements}

The authors gratefully acknowledge the collaborative environment at SparseTech that made this research possible. The theoretical and computational developments presented in this paper are part of an ongoing SparseTech research initiative on multi-teacher ensemble distillation for large language models. Patent Pending.

\bibliographystyle{plain}
\bibliography{../sparsetech_references}

\end{document}